\documentclass[letterpaper, 10 pt, conference]{ieeeconf}  

\IEEEoverridecommandlockouts                              

\let\labelindent\relax
\usepackage{mathtools}
\newtheorem{theorem}{Theorem}[section]
\newtheorem{lemma}{Lemma}[section]
\newtheorem{remark}{Remark}[section]
\newtheorem{definition}{Definition}[section]
\usepackage{url}
\usepackage{enumitem}
\usepackage{hyperref}
\newlist{compactenum}{enumerate}{4}
\setlist[compactenum,1]{nolistsep}
\setlist[compactenum,1]{nolistsep,label=(\alph*)}
\usepackage{array}

\usepackage{type1cm}        
\usepackage{makeidx}         
\usepackage{graphicx}        
\usepackage{multicol}        
\usepackage[bottom]{footmisc}

\usepackage{color}
\DeclarePairedDelimiter{\abs}{\lvert}{\rvert}
\DeclarePairedDelimiter\norm{\lVert}{\rVert}%
\usepackage{newtxtext}       %
\usepackage[varvw]{newtxmath}       
\usepackage{subfig}
\usepackage{todonotes}

\newcommand{\RvIcra}[1]{{\color{black}{{#1}}}}
\newcommand{\RvTypoIcra}[1]{{\color{black}{{#1}}}}

\title{\LARGE \bf
A Complete Set of Connectivity-aware Local Topology Manipulation Operations for Robot Swarms
}

\author{Koresh Khateri*$^{1}$, Karthik Soma*, Mahdi Pourgholi$^{2}$, Mohsen Montazeri, \\ Lorenzo Sabattini$^{3}$ and Giovanni Beltrame$^{1}$%
\thanks{$^{1}$ \'Ecole Polytechnique de
Montr\'eal, 2900 Boul \'Edouard-Montpetit, Qu\'ebec, CA
{(email: \tt\small\{koresh.khateri, karthik.soma, giovanni.beltrame\}@polymtl.ca})}%
\thanks{$^{2}$ Shahid Beheshti University, Tehran, Iran}%
\thanks{$^{3}$ University of Modena and Reggio Emilia, Reggio Emilia, Italy}%
}

\begin{document}

\maketitle
\def\thefootnote{*}\footnotetext{These authors contributed equally to this work} 
\thispagestyle{empty}
\pagestyle{empty}


\begin{abstract}
  The topology of a robotic swarm affects the convergence speed of
  consensus and the mobility of the robots.
  In this paper, we prove the existence of a complete set of local topology
  manipulation operations that allow the transformation of a swarm topology. The
  set is complete in the sense that any other possible set of manipulation
  operations can be performed by a sequence of operations from our set. The
  operations are local as they depend only on the first and second hop
  neighbors' information to transform any initial spanning tree of the network's
  graph to any other connected tree with the same number of nodes. The
  flexibility provided by our method is similar to global methods that require
  full knowledge of the swarm network. We prove the existence of a sequence of
  transformations for any tree-to-tree transformation, and derive sequences of
  operations to form a line or star from any initial spanning tree. Our work
  provides a theoretical and practical framework for topological control of a
  swarm, establishing global properties using only local information.
        
        
\end{abstract}

\section{INTRODUCTION} 
\label{sec:intro}

According to Cayley's formula~\cite{shor1995new}, there are $N^{N-2}$ possible
labeled spanning trees for a swarm network with $N$ robots. Each of these
determines different properties (e.g., coverage area, consensus rate, and
mobility) of the swarm. In this paper, we investigate the possibility of
transforming any initial spanning tree of the swarm's \RvTypoIcra{network, by only} using
local operations, to any other connected tree spanning all nodes of the system.
This transformation is of importance \RvTypoIcra{because a fixed spanning tree restricts the relative movement of a
swarm.}

In a fixed topology, the mobility of the robots is constrained since initial
neighbor robots have to remain neighbors throughout the mission even though they
might be required at different locations that are farther than the
communication range of the robots, \RvTypoIcra{whereas they could break their connection and still be connected
with a multi-hop path.}

The existence of a communication path between any pair of nodes in a robotic
network defines the connectivity of that network. Continuous connectivity (i.e.,
strict connectivity from the start of the mission to the end) between the robots
of a swarm is essential in several applications \cite{liu2021distributed,dong2018consensus} Assuming that communication is limited by distance, the existence of a
communication link between two robots depends on their position. Therefore,
robot mobility causes the topology to dynamically change, potentially breaking
links. 

\RvIcra{Connectivity maintenance is a supervisory control that blocks disconnecting action from a primary controller that pursues the primary goal of the swarm. In \cite{capelli2020connectivity} a control barrier function based controller minimally tweaks the primary controller's command such that connectivity is guaranteed.} We propose a set of operations to transform a spanning tree topology into any
other spanning tree topology using only local, connectivity-aware operations
which require data from immediate and 2-hop neighbors. \RvIcra{Then the primary controller can use these operations to have more flexibility when topology manipulation is required}. Note that data is not
propagated from farther neighbors using decentralized averaging techniques \RvTypoIcra{like~\cite{zhang2020distributed}}. Although similar local
manipulation operations have been proposed in \RvTypoIcra{the}
literature~\cite{yi2021distributed,varadharajan2020swarm,khateri2020connectivity,majcherczyk2018decentralized},
this is the first local complete set without using any global index of connectivity, to the best of the authors' knowledge.

\begin{figure}[tbp]
  \centering
  \includegraphics[width=0.45\columnwidth]{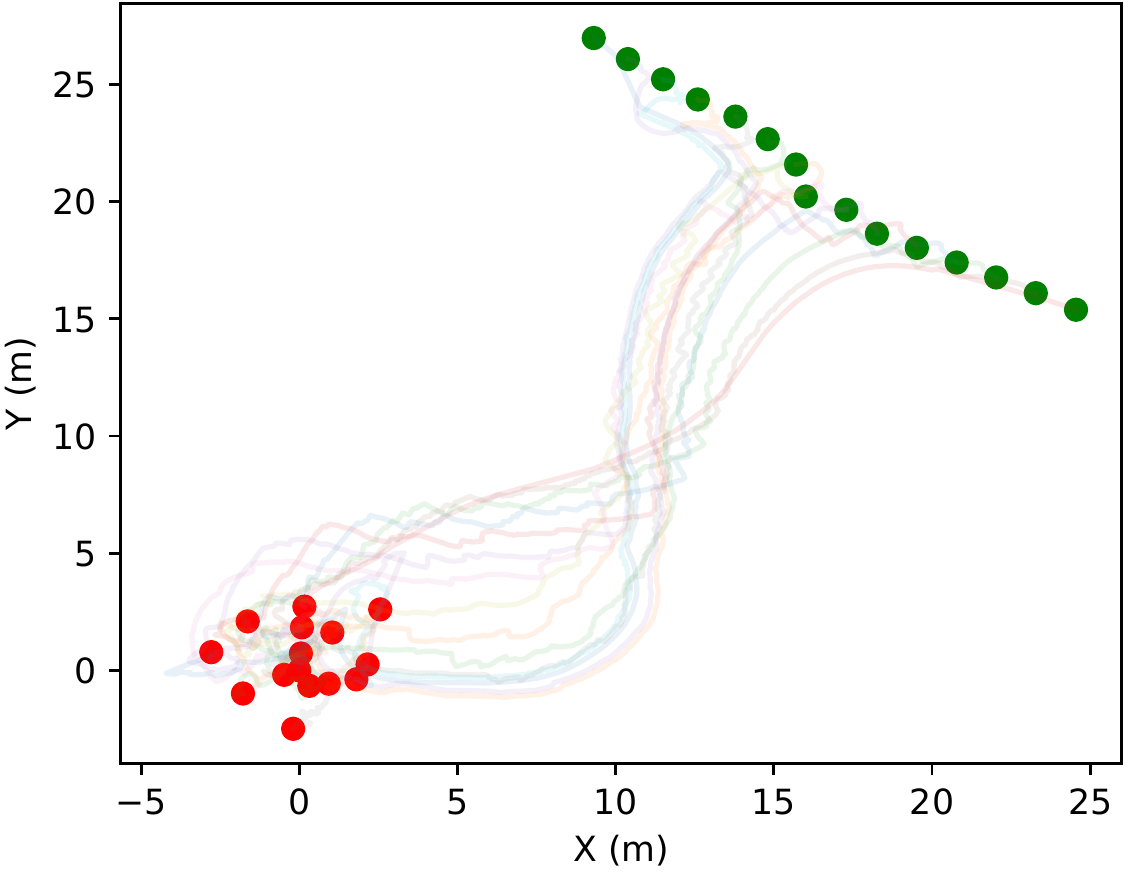}
  \includegraphics[width=0.45\columnwidth]{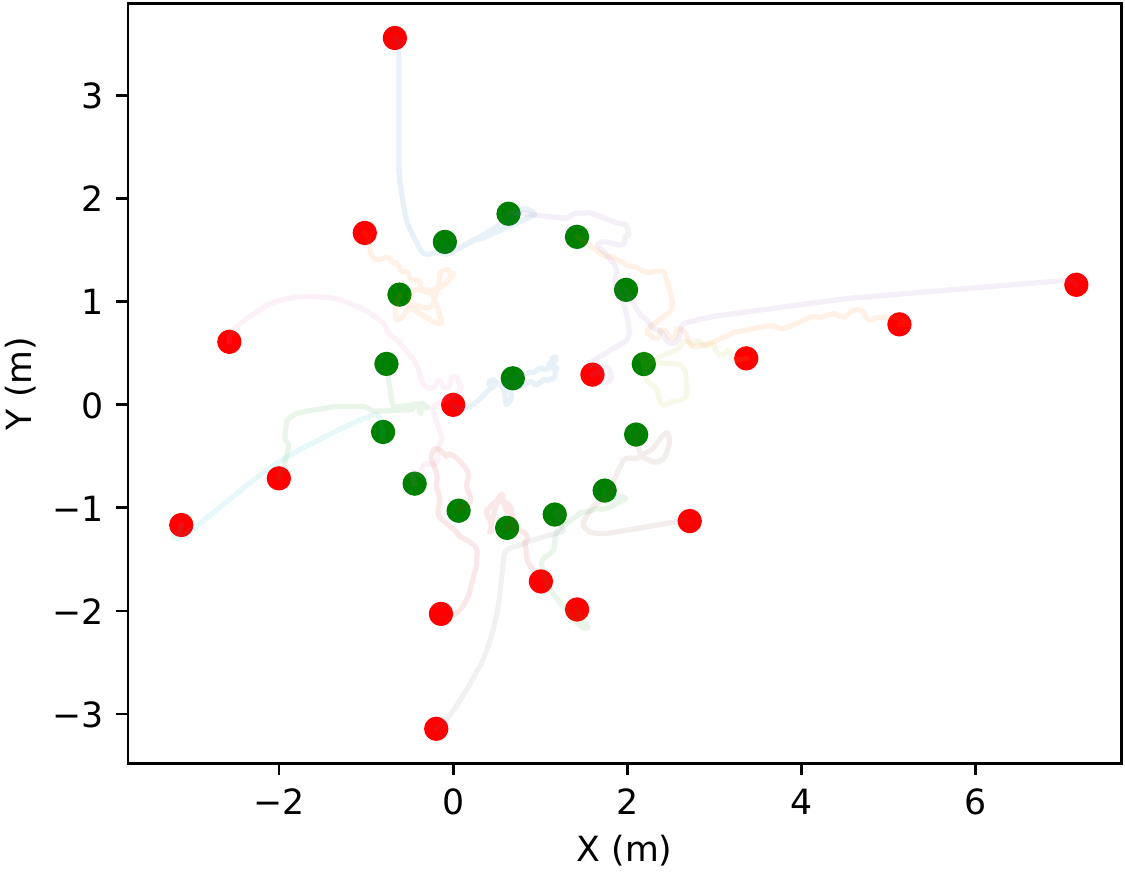}
  \includegraphics[width=0.45\columnwidth]{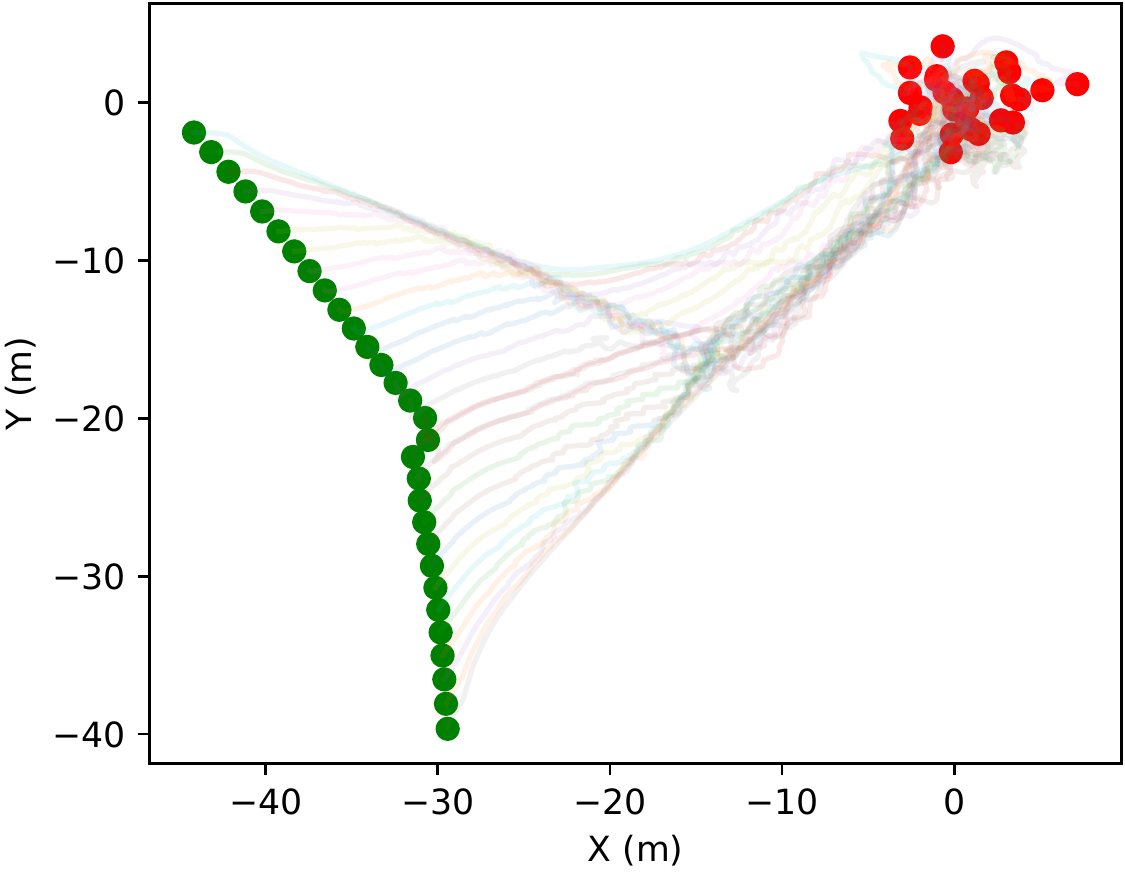}
  \includegraphics[width=0.45\columnwidth]{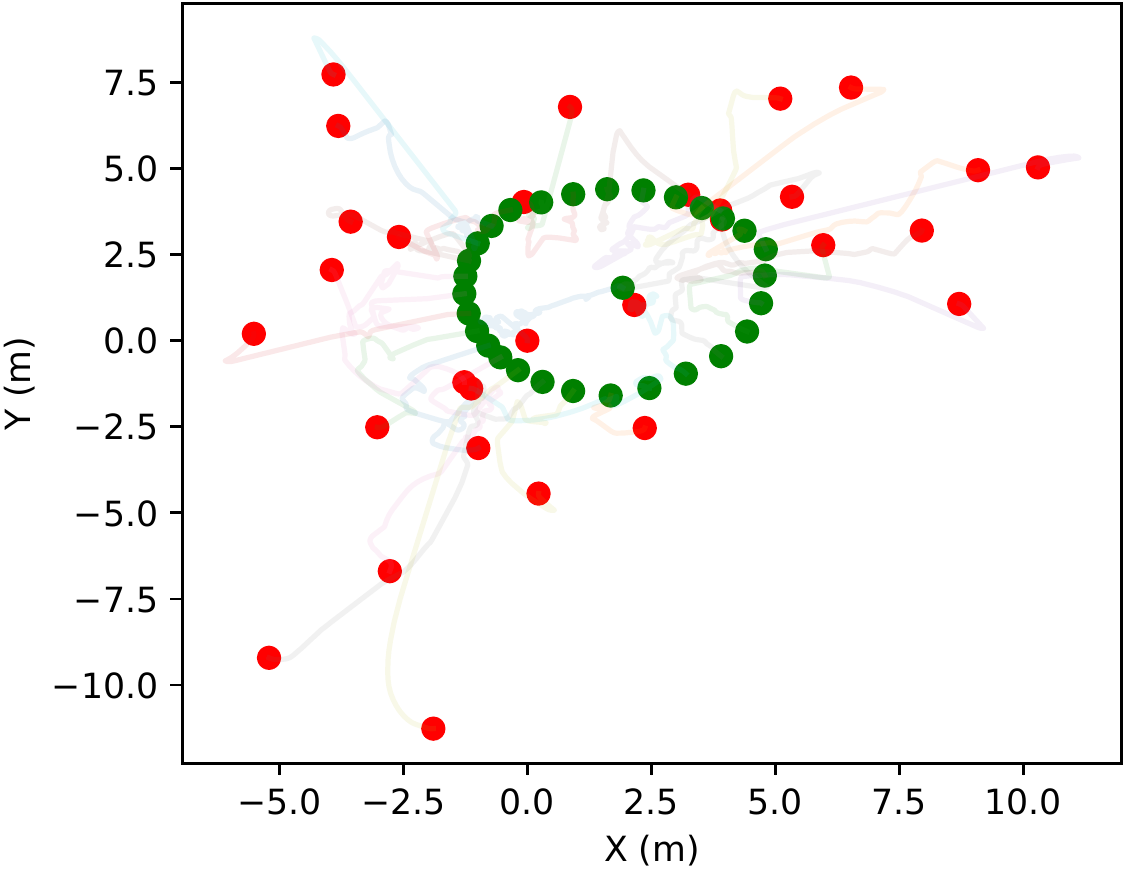}
  \includegraphics[width=0.45\columnwidth]{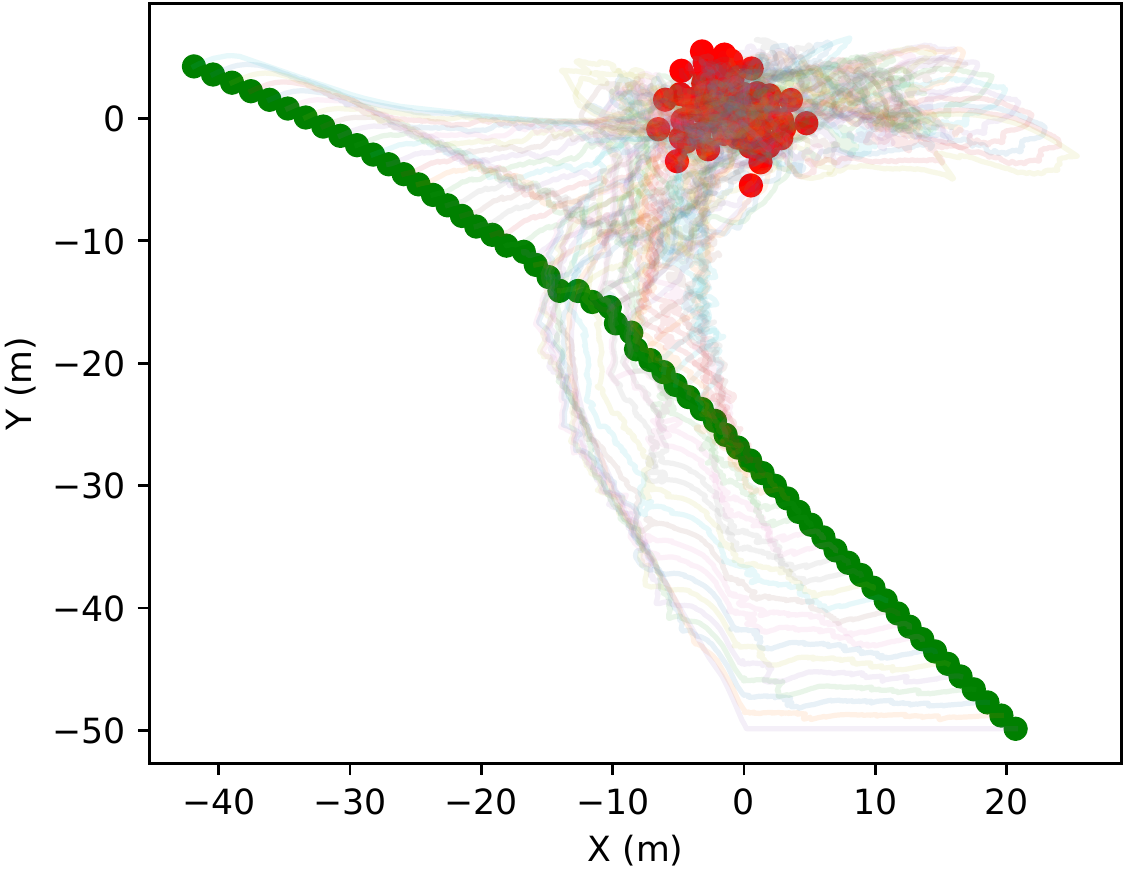}
  \includegraphics[width=0.45\columnwidth]{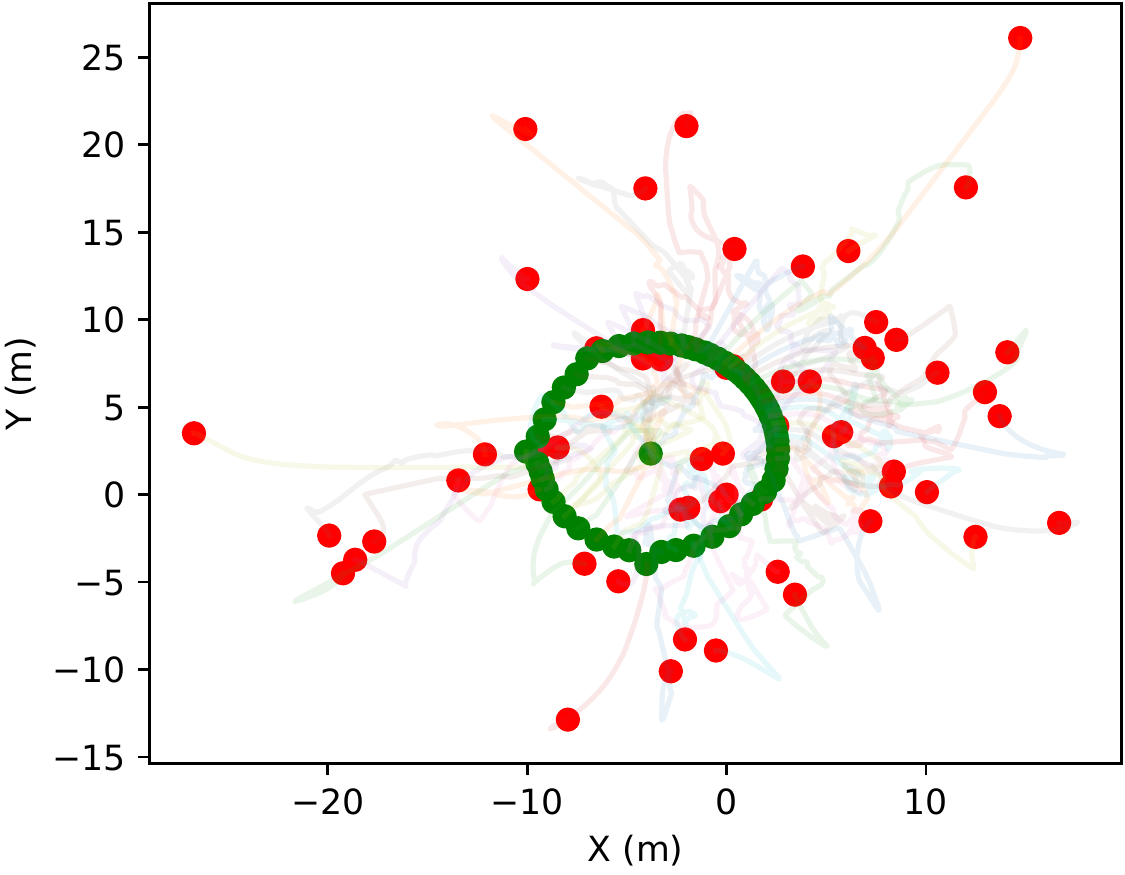}
  \caption{Trajectory visualizations for line and star formations for 15, 30, and 60 robots, column-wise. Initial positions and final positions are represented by red and green colors, respectively.}
\label{fig:introduction_figure}
\end{figure}%

\section{RELATED WORK}
The communication topology of a swarm can be fixed or vary throughout its
mission. When enforcing a fixed topology, also known as local connectivity
maintenance, every initially existing link between robots is maintained (e.g. by
the use of virtual potential functions or control barrier functions) by
enforcing an inter-robot distance that is smaller than the robots' communication
range~\cite{khateri2019comparison}. Local connectivity maintenance is a simple
approach that requires minimal communication overhead and velocity constraints,
but it restricts the relative mobility of two initially neighboring robots and
makes it impossible to change the global topological properties of the swarm (e.g.
consensus rate or coverage). With topology manipulation, the connectivity of the
network must be taken into account either continuously from start to end
\cite{sabattini2013distributed,dimarogonas2008decentralized,ajorlou2010class} or
it has to be re-established intermittently
\cite{saboia2022achord,kantaros2017distributed}.

Preserving connectivity has been an active research area, this subject is
relevant in achieving almost all cooperative tasks in multi-robot and
multi-agent systems such as flocking
\cite{liang2021decentralized,feng2019connectivity}, search and rescue
\cite{flushing2013connectivity,hayat2017multi} and formation control
\cite{fu2020distributed,peng2019path}.

Published works on continuous connectivity maintenance can be classified into three main approaches:
\begin{itemize}

\item Local connectivity
  maintenance~\cite{wen2012connectivity,ji2007distributed,hsieh2008maintaining},
  which designs the control plans to avoid any disconnection of initially
  existing communication links;
    
\item Global connectivity maintenance
  \cite{gasparri2017bounded,siligardi2019robust,sabattini2013distributed,capelli2020connectivity},
  where the breaking of communication links is allowed, as long as the overall
  communication topology remains connected. Global connectivity maintenance is
  more flexible: it allows reconfiguring the network's graph to any
  possible connected graph with the same number of vertices. The criterion for
  allowing an edge disconnection is estimated or calculated based on a global
  index which requires complete information over the graph. In the decentralized
  \RvIcra{version of this approach interactions are local, such that} iterative distributed local averaging or consecutive local broadcast
  is used to collect information in order to estimate this global index.
  
\item Connectivity-aware local topology manipulation operations: these methods
  are enhancing the local connectivity maintenance method by performing local
  topology manipulations while maintaining connectivity. A node permutation
  method to locally swap neighbors is proposed in
  \cite{khateri2020connectivity}. A chained swarm relay method to provide
  maximum coverage from a ground station is investigated in
  \cite{varadharajan2020swarm}. A logical tree acts as a communication backbone
  in \cite{majcherczyk2018decentralized} where two topology manipulation
  operations called outward and inward are utilized to add robots to this
  backbone tree until they reach predefined targets. However, these methods do
  not provide the same flexibility compared to global connectivity maintenance methods.
\end{itemize}

\RvIcra{From the literature on theoretical distributed dynamic networks, an approach~\cite{Michail2020} similar to our star formation algorithm is proposed for reconfigurable robots. In~\cite{michail2019} the authors propose the use of rotation and slide operations which are similar to our to be defined leafization and leaf transfer operations for reconfigurable robots arranged in 2-dimensional grids. They have shown these two operations are universal in a centralized mode as the operations can make changes to the topology in the x and y axes respectively. They also propose a line formation algorithm that unlike our method requires global information.}

In this paper, we investigate the possibility of defining connectivity-aware
local topology manipulation operations in such a way as to provide the same
flexibility as with global methods. We demonstrate that there is a complete
set of operations that can manipulate the swarm's
topology in any possible connected form.

The main reason for developing connectivity-aware local topology manipulation
operations is the fact that the decentralized estimation of global connectivity
indices is not scalable and it is very sensitive to
noise~\cite{varadharajan2020swarm,khateri2019comparison}. Despite the
flexibility that a global method provides, the convergence time needed to
estimate a global index restricts the speed of
robots~\cite{khateri2019comparison}. This phenomenon is caused by practical
constraints like finite bandwidth and/or communication delays. The superiority
of global methods in some
works~\cite{sabattini2013decentralized,sabattini2013distributed} is due to
assuming continuous, infinite bandwidth and cost-less communication. It is
worth noting that in a global method communications are on $(N-1)-$hop path, in
the worst case. Thus, it can have a significant impact on a large-scale network.
However, in a local method communication paths are one or two hops. Therefore,
the communication delay is insignificant and can be ignored. In a real-world
situation, there is a trade-off between graph flexibility and the maximum
allowed speed of robots.

The algebraic connectivity which is the second smallest eigenvalue $\lambda_2$
of the Laplacian matrix $L$, associated with the communication graph, is a
connectivity index that is employed to indicate the connectivity of the network.
$\lambda_2$ is shown to be a concave function of the Laplacian
matrix~\cite{merris1994laplacian}. Therefore, optimization approaches to
calculating the control inputs had been proposed in order to maximize it
\cite{de2006decentralized,kim2005maximizing}. However, calculating the eigenvalues of a
matrix requires full knowledge of the matrix. A centralized approach to
calculate eigenvalues and eigenvectors is used in \cite{kim2005maximizing}.
Power iteration is used in designing a decentralized algorithm to estimate
algebraic connectivity in \cite{yang2010decentralized}. Based on this
decentralized estimator, several global connectivity maintenance methods to
apply control inputs such that $\lambda_2$ is preserved positive are proposed
\cite{sabattini2013distributed,sabattini2013decentralized,secchi2012decentralized}.
A connectivity maintenance supervisor considering time delays using control barrier
functions to preserve $\lambda_2$ more than a threshold is proposed in
\cite{capelli2021decentralized}.

Another approach to global connectivity maintenance is by designing a controller
such that only links with an alternative $k-$hop path are allowed to disconnect.
The length of a path in a graph with $N$ nodes is less than $N-1$. Thus, global
connectivity of the network could be maintained with $k=N-1$
\cite{zavlanos2011graph} and $k=1$ corresponds to local connectivity. In
\cite{zavlanos2005controlling} a centralized controller is proposed to maintain
alternative paths of different sizes. Based on this view, authors in
\cite{khateri2020decentralized} have proposed a decentralized version to
preserve $k$-hop connectivity that uses a modified dynamic source routing method
to find alternative paths with a predetermined length. Thus, the locality level
of the search for an alternative path is selected by the designer.

\section{PRELIMINARIES}
\label{sec:prelim}
A limited connectivity range swarm of $N$ robots is modeled with a distance-dependent graph. Let $G(t)=(V_G,E_G(t))$ be a graph with the set of vertices
$V_G=\{v_1,\cdots,v_N\}$ and the edges $E_G(t)=\{(i,j) \in V_G \times V_G \mid
A^G_{i,j}(t)>0 \}$ where $A^G(t)$ is the adjacency matrix. \RvIcra{In this paper, $A^G_{i,j}(t)=1$, if there}
is an edge $(i,j)$ at time $t$. Otherwise, $A^G_{i,j}(t)=0$. A graph $G$
is considered as limited range if $A_{i,j}=0$ when $\norm{x_i-x_j}>R$, where $x_i$, $x_j$ and $R$
are the position of nodes $i$, $j$, and the communication range,
respectively. The neighbors of node $v_i$ at time $t$ are $N^G_i(t)=\{v_j \in V_G
\mid A^G_{i,j}(t)>0 \}$.

As in \cite{MesbahiEgerstedt+2010}, the degree matrix associated with a graph $G$ is defined as:
\begin{equation}\label{eq:Degree}
    D^G_{i,j}(t)=\begin{cases}
    \sum_{k \in N_i} A^G_{i,k}(t) & i=j\\
    0 & \text{otherwise}
  \end{cases}
\end{equation}
and the Laplacian matrix is $L^G(t)=D^G(t)-A^G(t)$. The second smallest
eigenvalue of the Laplacian matrix $\lambda_2$ is an index of the graph
connectivity. The degree of a node $\text{degree(node)}$ is defined as the
number of edges connected to that node.

$\text{Vert(G)}$ is a function that returns the vertices of a graph.

A tree $T(t)=(V_T,E_T(t))$ is a graph in which any two vertices are connected by
exactly one path. A spanning tree of a connected graph $G$ is a subgraph of $G$,
which is a tree and contains all vertices of $G$. Let "$\setminus$", "$\abs{.}$"
be the set complement operator and the size of a set, respectively. A leaf of a
tree $T$ is a node with only one edge connected to it. A super-leaf is a
sub-tree $S_j$ with $j$ as its root where, $A_{i\in \{V_{S_j}\setminus j\},k
  \in \{V_T\setminus V_{S_j}\}}=0$ and $\abs{N_j\setminus V_{S_j}}=1$.
Which means nodes in $V_{S_j}\setminus j$ do not have any edge with nodes in
$V_T\setminus V_{S_j}$ and node $j$ is connected to only one of the nodes
in $ V_T\setminus V_{S_j}$.

The following statements are equivalent for an undirected graph:
\begin{enumerate}[label=(\alph*)]
\item The graph is connected.
\item There is at least one path between each pair of nodes in the graph.
\item $\lambda_2>0$.
\item The graph has at least one spanning tree.
\end{enumerate}

\section{MAIN RESULTS} \label{sec:main}

In this section, \RvIcra{we only deal with the existence of a complete set of local operations. The implementation of such operations is shown for the case of line formation and star formation in the supplementary material~\cite{suppmaterial}.} We define our set of operations that are able to manipulate
connected trees to transform them into any other possible connected tree of the
same size. Through the unique Pr\"ufer bijection of trees to strings, we show
the completeness of our set of operations. A Pr\"ufer sequence is a bijection
between the set of labeled connected trees on $N$ vertices with the set of
sequences of length $N-2$. The manipulation operations are provided to transform
the Pr\"ufer sequence of the tree to any other Pr\"ufer sequence of length
$N-2$. Lemma~\ref{lem:transformation of all trees} proves the possibility of
transforming a spanning tree into any other spanning tree of the same size.
Finally, Theorem~\ref{th:reducibility} shows equality, with respect to
the freedom of movement of robots, of the proposed method with the traditional
global connectivity maintenance methods.

Consider labeling nodes of a tree with natural numbers $i \in \mathbb{N}$,
assume the leaf with the lowest number to be $l_1$ and let the node connected to
$l_1$ be node $a$, note that because $l_1$ is a leaf, $a$ is the unique node
connected to $l_1$, then consider that after deleting $l_1$, $l_2$ is the leaf
with the lowest label number, which is connected to node $b$. By repeating this
process until there remain only two nodes, a sequence of labels of the connected
nodes to the lowest leaves could be created, which is called the Pr\"ufer
sequence of the tree (i.e. $\{a,b,\cdots\}$ here).

\begin{lemma} \label{lem:Prufer} \cite{prufer1918neuer} Every labeled tree on
  $N$ nodes has one and only one Pr\"ufer sequence of length $N-2$.
\end{lemma}

In this paper, we have defined these four topology manipulation operations for
changing the topology of a tree as follows:

\begin{definition}
Leafization $L(j,k)$: this operation is to transform a non-leaf node into a leaf node. Consider node $j$ to be a non-leaf node with a neighbor $k \in N_j(t)$ and $\abs{N_j(t)}>1$. With the leafization operation, all the edges $\{(j,p\in \{N_j\setminus k\})\}$  will be disconnected and new edges $\{(k,p\in \{N_j\setminus k\})\}$ will be established.
\end{definition}    
\begin{definition}
Leaf transfer $LT(l,j,k)$: This operation is to transfer a leaf from one node to another node. Consider node $l$ as a leaf connected to node $j$ while node $j$ is a neighbor of node $k$. Leaf transfer of leaf $l$ from node $j$ to node $k$ is to disconnect the edge $(l,j)$ and establish a new link $(l,k)$.
\end{definition}    
\begin{definition}
Super leafization $SL(S_j,k)$: Consider $S_j$ to be a sub-tree rooted in $j$ (i.e. $A_{i\in \{Vert(S_j)\setminus j\},k \in  \{Vert(T)\setminus Vert(S_j)\}}=0$, where $A_{i,k}$ means the $i,k$th element of the adjacency matrix). Consider node $k$ such that $k \in \{N_j\setminus Vert(S_j)\}$ and $\abs{N_j\setminus Vert(S_j)}>1$. With the super-leafization operation, all the edges $\{(j,p\in N_j\setminus{\{k\cup Vert(S_j)}\})\}$ will be disconnected and new edges $\{(k,p\in N_j\setminus\{k\cup Vert(S_j)\})\}$ will be established.
\end{definition}   
\begin{definition}
Super leaf transfer $SLT(S_l,j,k)$: This operation is to transfer a super leaf connected to one node to another node. Consider super node $S_l$ as a super leaf connected to node $j$ while node $j$ is a neighbor of node $k$. Super-leaf transfer of $S_l$ from node $j$ to node $k$ is to disconnect the edge $(l,j)$ and establish a new link $(l,k)$.
\end{definition}

\begin{remark}
The $SL(S_j,k)$ and $SLT(S_l,j,k)$ operations are generalizations of $L(j,k)$ and $LT(l,j,k)$ operations, respectively. \RvIcra{$SL(S_j,k)$ and $SLT(S_l,j,k)$ are still local operations as only the edges between the root of the super node and the neighbors are affected. All the other edges in the super nodes are left untouched.}
\end{remark}

An example of these operations is illustrated in Figure~\ref{fig:Tree transform operations}.   
    \begin{figure} 
  \centering
  \subfloat[]{
       \includegraphics[width=.45\linewidth]{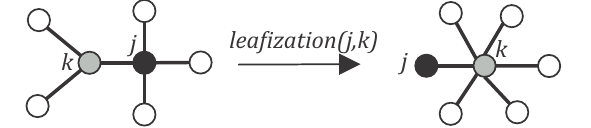}}    
  \subfloat[]{
        \includegraphics[width=.45\linewidth]{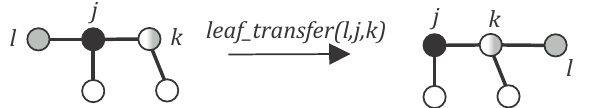}}
    \\
      
     \subfloat[]{
        \includegraphics[width=.45\linewidth]{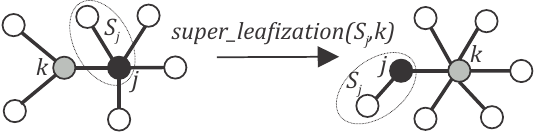}}  
    \subfloat[]{
        \includegraphics[width=.45\linewidth]{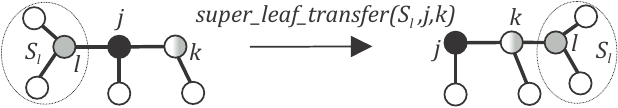}}

  \caption{(a) $L(j,k)$, (b) $LT(l,j,k)$, (c) $SL(S_j,k)$, and (d) $SLT(S_l,j,k)$ operation.}
  \label{fig:Tree transform operations} 
\end{figure}

\begin{lemma}\label{lem:connectivity of transformation} Under the defined tree manipulation operations the result remains a tree. \\
\begin{proof}
Regarding the $L(j,k)$ operation: Let $N_j(0)=\{j_1,j_2,\cdots\}$ be the set of neighbors of node $j$ before the operation. Consider a pair of nodes $s$ and $d$ that are connected by a path including node $j$, before the operation. Let the unique initial path be $\{s,\cdots,j_x,j,j_y,\cdots,d\}$ then, after the leafization operation, this pair of nodes are connected uniquely, by the path $\{s,\cdots,j_x,k,j_y,\cdots,d\}$. Because, with the operation all the neighbors of $j$ will be connected to $k$ as defined in the operation's definition and $\{j,k\}$ is the unique path between nodes $j$ and $k$. Thus the path $\{s,\cdots,j_x,k,j_y,\cdots,d\}$ is the only path between $s$ and $d$. Hence, there is a unique path between each pair of nodes after the operation. Therefore, the initial-tree remains a tree and is connected.

The proof for $LT$ operation is trivial and the proof for $SL$ and the $SLT$ operation are omitted, as they are generalizations of the $L$ and $LT$ operations.
\end{proof}
\end{lemma}
\RvIcra{Please note that the following lemma is only showing the existence of a series of operations for transforming any tree and the Pr\"ufer sequence and labeling nodes are only used for proving the lemma. For implementing our method, we need to design algorithms based on our manipulation operations like what we have illustrated in supplementary material~\cite{suppmaterial}}
\begin{lemma}
\label{lem:transformation of all trees} 
Every tree in the set of all possible connected trees with $N$ nodes is transformable to any other connected tree of the same size, by using the defined topology manipulation operations.
\end{lemma}
\begin{proof}
Let us call the initial tree, resulting-tree($m$), and the final tree as the tree we start from, the result after $m$ steps and the tree we want to end up with, respectively.
Consider the final-tree's Pr\"ufer sequence to be $\{p_1,p_2,\cdots,p_{N-2}\}$ and the sequence of leafs with lowest label while composing this Pr\"ufer sequence be $\{l_1,l_2,\cdots,l_{N-2}\}$. Beginning with any arbitrary initial tree, if $l_1$ is not a leaf of the initial tree, it could be transformed into a leaf, by using the $\text{leafization}$ operation. In a connected tree, there is a unique path between any pair of nodes. Thus, there is a unique path between $l_1$ and node $p_1$. Then, using a sequence of leaf transfer operations, $l_1$ can move through that path to establish an edge with $p_1$. After this step, $p_1$ will be the first element of the final-tree's Pr\"ufer sequence, as desired. At step $m$, if the neighbors of $l_m$ in the resulting-tree($m-1$) (i.e., $N_{l_m}(m-1)$ ) contains any node in the set $\{l_1,\cdots,l_{m-1}\}$ (i.e., $N_{l_m}(m-1)\cap\{l_1,\cdots,l_{m-1}\}\neq \emptyset$) then consider the subtree, containing $l_m$ and all the subtrees rooted in nodes $N_{l_m}(m-1)\cap\{l_1,\cdots,l_{m-1}\}$, if this subtree is not a super leaf of the resulting-tree($m-1$), it could be transformed into a super leaf, by using the $\text{super\_leafization}$ operation. Then, there is a unique path between $l_m$ and node $p_m$. Thus, using a sequence of super leaf transfer operations, $l_m$ can move through that path to establish an edge with $p_l$. This makes $p_m$ the $m$th item in the final-tree's Pr\"ufer sequence.  
As stated in Lemma~\ref{lem:Prufer}, the complete set of trees on $N$ nodes are uniquely described by Pr\"ufer sequences. Therefore, the possibility of transforming an arbitrary initial tree into a tree with a desired Pr\"ufer sequence by only using certain operations means the possibility of achieving any desired tree using those operations. 
\end{proof}

The following theorem shows that all achievable benefits from global connectivity maintenance are possibly achievable using a modified local method, in which the initial tree is maintained by local maintenance of the initial edges of the tree, while edges are only allowed to break or be created by using the topology manipulation operations.  
  
\begin{theorem} \label{th:reducibility}
Global connectivity maintenance, for distance-dependent networks, is reducible to local topology manipulation with defined tree manipulation operations.
\begin{proof}
Global connectivity maintenance needs an initially connected graph. A spanning tree of this graph could be selected as the initial tree to be maintained connected in the local topology manipulation scheme. At any instance or after reaching any desired configuration (i.e breaking edges and creating new ones), with the global connectivity maintenance, a new connected graph is connecting the nodes of the network, which also contains a spanning tree. This former spanning tree is reachable from the initial tree with our topology manipulation operations, according to Lemma~\ref{lem:transformation of all trees}, and it is connected, based on Lemma~\ref{lem:connectivity of transformation}. Therefore, any configuration reachable with global connectivity maintenance is also achievable with local tree maintenance with tree manipulation operations.

\end{proof}
\end{theorem}

\begin{remark}
The tree manipulation operations need information, only from immediate (i.e. 1-hop) neighbors and 2-hop neighbors. Therefore, the operations are considered to be local.
\end{remark}

\begin{remark}  
Theorem \ref{th:reducibility}, only shows the possibility of the reduction of global connectivity maintenance to a modified local connectivity maintenance. It does not provide a sequence of operations needed to transform an initial tree into a desired final tree. To find the sequence in which the operations should be done we need to have the target spanning tree be definable with local rules which do not require any global information. Otherwise, a one-time (i.e., not continuously every moment during the mission and only when a change of topology is needed) propagation of the structure of the target tree is needed which can be done decentralized. This point is explained more for the case of line and star topologies which are locally definable in Section~\ref{sec:Exp} and supplementary materials~\cite{suppmaterial}.
\end{remark}



To implement the topology manipulation operations, it is needed to establish new connections with second neighbor robots. During the operations, Due to the communication range of robots, as considered to be $R$,  the neighbor's distance must be reduced until it reaches $\frac{R}{2}-\delta$, $\delta$ being a positive small constant. After this distance reduction, robots will be able to communicate and establish new connections with their second neighbors. A method to perform such a local inter-distance reduction using gradients of potential functions is investigated in \cite{khateri2020connectivity} for a local swap operation. We have used the same controller and the proof of connectivity preservation during the inter-distance reduction is similar.

\section{EXPERIMENTS}
\label{sec:Exp}

\RvIcra{In this section, we show how we can use the operations defined in Section~\ref{sec:main} to perform line and star formation from any initially connected random topology. The knowledge of Pr\"ufer sequences and labeling of nodes are not used for forming lines and star formation. We choose the line and star formation because the sequences required for a line/star topology from any initial topology can be determined by local information such as the number of neighbors of each node and the roles of the neighbor robots (free to participate in an operation or busy). Unlike the Theorem ~\ref{th:reducibility} we use parallel operations to make faster lines and stars. In particular, to form lines we use leaf transfers and super leaf transfers, and to form stars we use leafizations and super leafization. For more information about the rules refer to the supplementary material~\cite{suppmaterial}. For other symmetric topologies that have a local structure, similar algorithms could be developed because the possibility of the existence of such an algorithm is proven in Lemma~\ref{lem:transformation of all trees}.}

we have designed these experiments in the Buzz~\cite{pinciroli2016buzz} programming language with a set of Khepera IV robots using ARGoS3~\cite{Pinciroli2012} as our physics-based simulator. Our experiments show that we can transform any initial topology into a line or a star, with swarms of varying sizes (i.e., number of robots), namely 15, 30, and 60. The robots have a communication range $R_{\text{range}}$. We run each experiment for 5 randomly generated initial positions. 

\RvIcra{We require the robots to be initially connected and to be aware of their immediate and second neighbors on an initial spanning tree.  Please note that this initial spanning tree could be given to all robots before they start their mission or it can be calculated in a decentralized way~\cite{majcherczyk2018decentralized}.}

After this step, the robots
perform local operations on the spanning tree to form
line or star topologies. When the robots have to perform the
transformation, the robots involved come to a reduced distance 
$R_{\text{transfer}}$. This gives the robots the ability to form connections with multiple neighbors and break links if needed. Robots maintain a distance $R_{\text{mission}}$ while not performing the operations. 


\RvIcra{The robots make use of range and bearing measurements to know the relative positions of their neighbors and use potential-based controllers similar to~\cite{wen2012connectivity}. This allows the robots to maintain the required distances $R_{\text{transfer}}$ and $R_{\text{mission}}$.}

Figure~\ref{fig:time_taken} shows the time taken to transform the topology for
line and star cases, labeled as topology transformed, and the time required to
arrange the robots, labeled as arrangement finished. This latter is the physical
positioning of the robots, e.g. the straightening of the line formation and the
equalization of angles between the points of the star~\ref{fig:introduction_figure}.

\RvIcra{ Figures~\ref{fig:line_formation} and \ref{fig:star_formation} depict $\lambda_2$ and coverage area. We plot the time evolution of $\lambda_2$ of the tree, as a connectivity index and a parameter specifying the consensus rate. This is done to demonstrate the connectivity awareness of our method which has to stay greater than zero over the experiment. 

In line formation, the value of $\lambda_2$ of the tree reduces with time and reaches a constant value when it is straightened out. Whereas in star formation, the value of $\lambda_2$ of the tree increases and reaches the same value no matter the number of nodes in the system. The coverage area has been shown to decrease for the star and to
increase for the line case. This shows the trade-off between $\lambda_2$ 
and the coverage area. This is why topology manipulation is required to provide flexibility.

We also provide the plots for $\lambda_2$ of the graph and some progress indices for line in the form of evolution of $\text{degree(robot)}=1$ and $\text{degree(robot)}=2$ with time and for star in the form of evolution of $\text{degree(robot)}=1$ and $\text{degree(robot)}\geq2$ with time in the supplementary material~\cite{suppmaterial}. This is to show that the manipulation operations are changing the initial topology closer to a line and star with time.}

It is worth noting the importance of choosing the two topologies line and star.
For the line, the main advantage is the coverage it can provide. This is useful
when the robots are searching, exploring, or they need communication with a base
station. For star, the main advantages are the ability to make decisions faster
(e.g., consensus).\RvIcra{ We show that from any random topology we can form the tree with maximum and minimum $\lambda_2$. Based on some environmental conditions or user requirements, the swarm can make a decentralized decision to change the topology of the swarm and consequently manipulate $\lambda_2$ (e.g., when the swarm wants to forage it can form a line as they need to maximize their coverage, and when they want to decide where to store food it can form a star as they need to make faster decisions.) }

\begin{figure}[tbp]
        \centering
        \includegraphics[width=0.45\columnwidth]{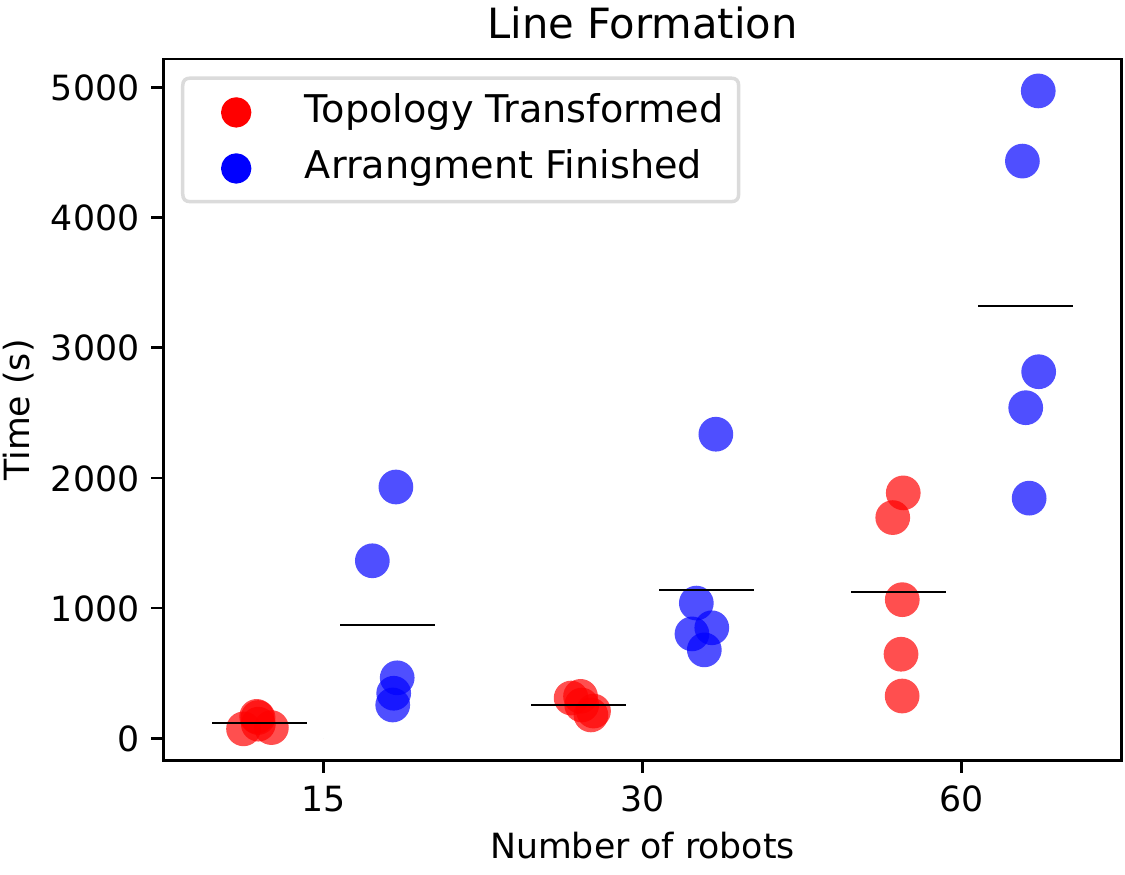}
        \includegraphics[width=0.45\columnwidth]{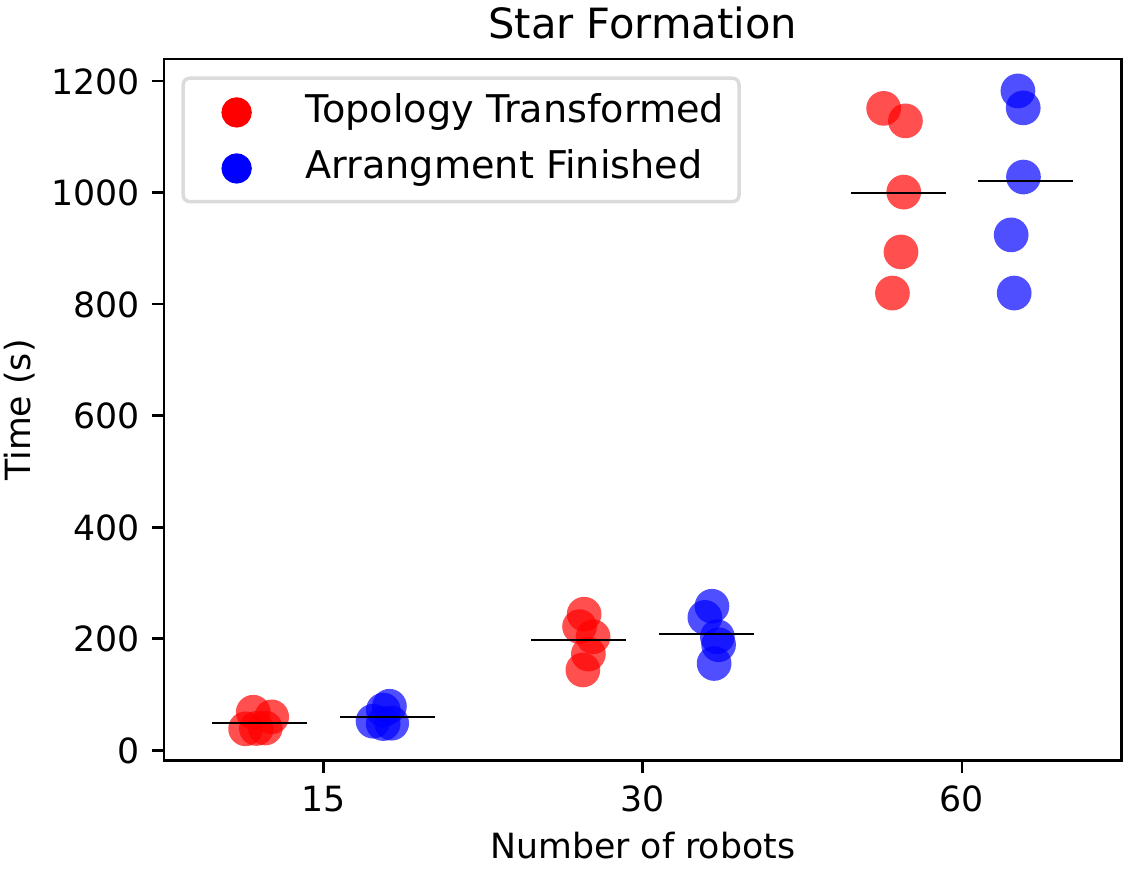}
        \caption{Plots for the time taken to complete the tree transformation and arrangement completion, from randomly initial trees to the line and star formation, for 15, 30, and 60 robots.}
    \label{fig:time_taken}
    \vspace{-2.0em}
\end{figure}%
\begin{figure*}[tbp]
    \centering
    \includegraphics[width=0.325\textwidth]{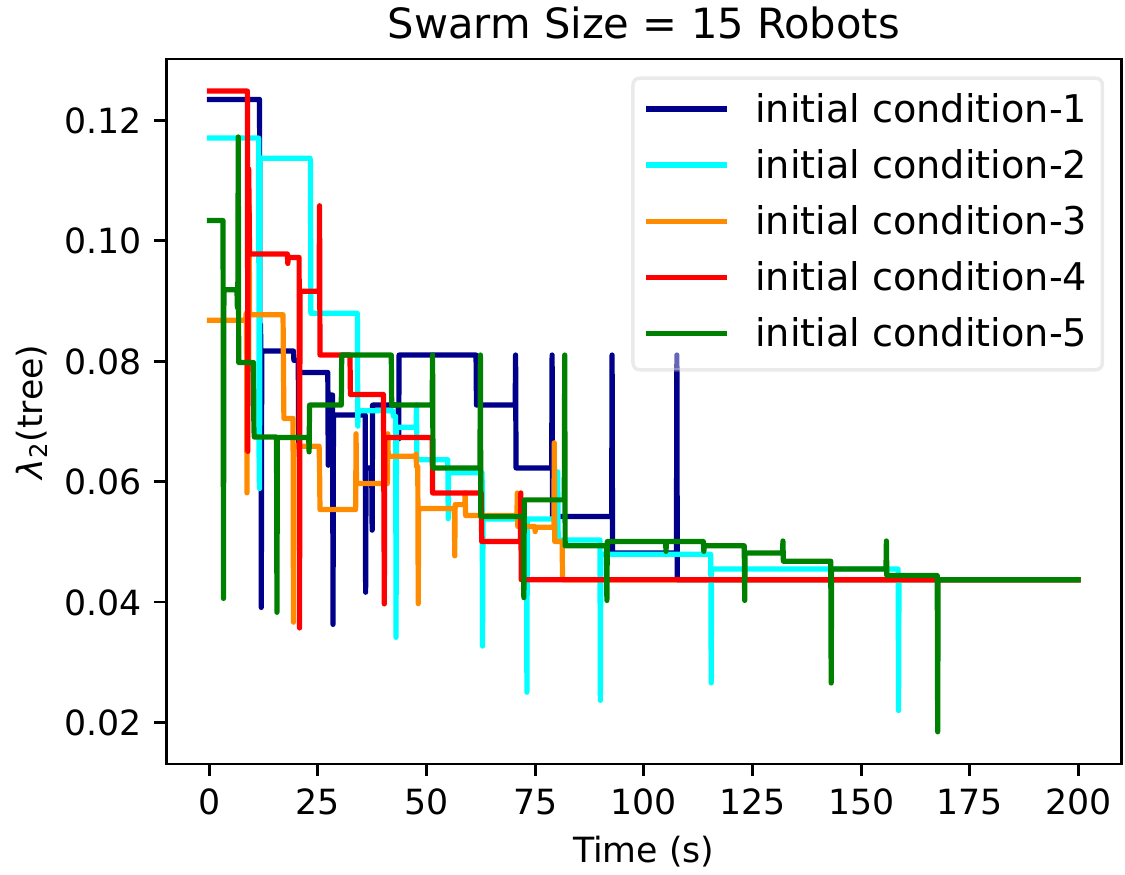}
    \includegraphics[width=0.325\textwidth]{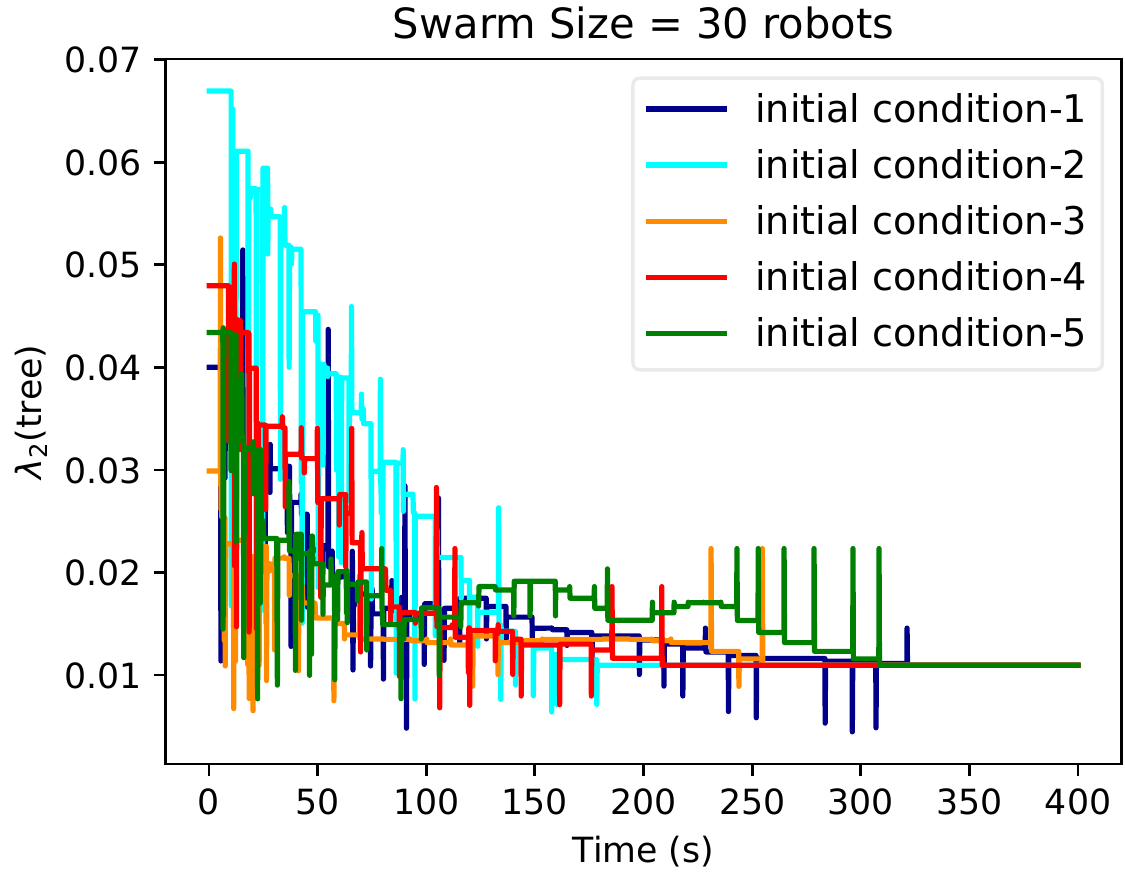}
    \includegraphics[width=0.325\textwidth]{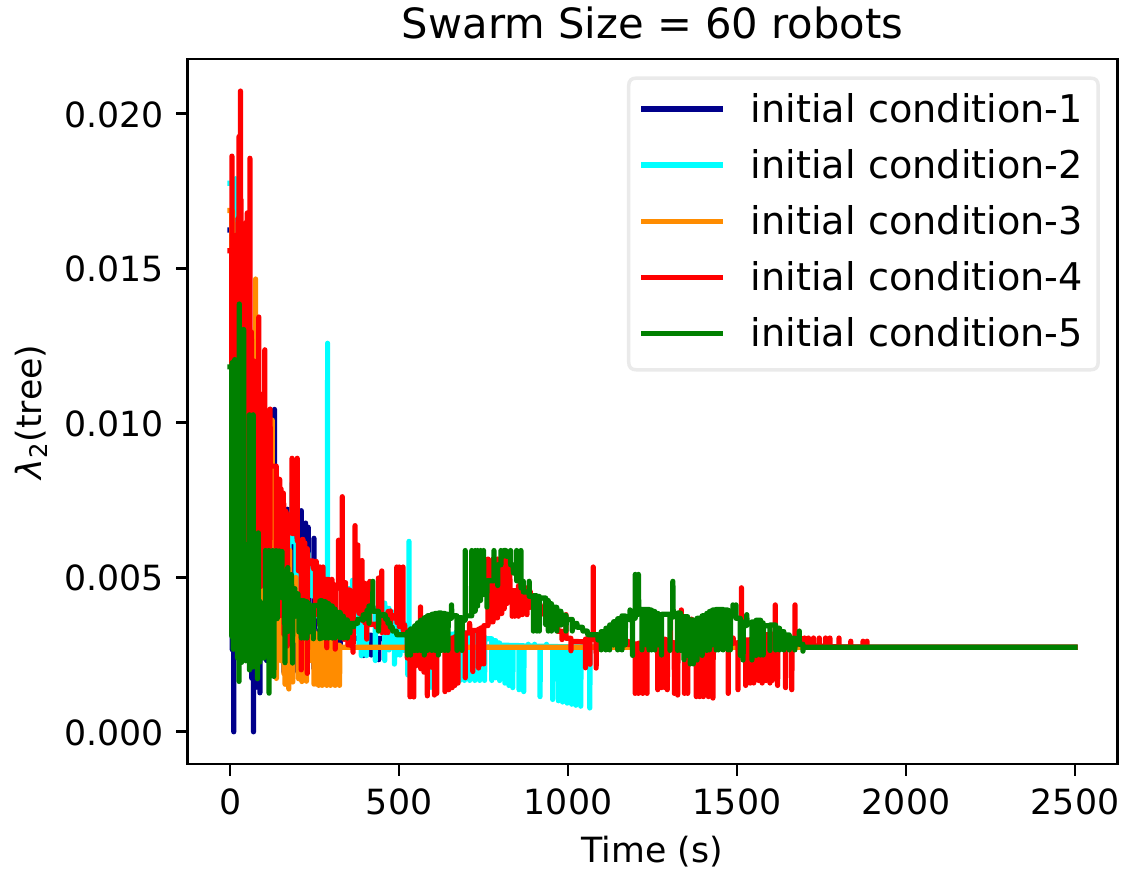}
    \includegraphics[width=0.32\textwidth]{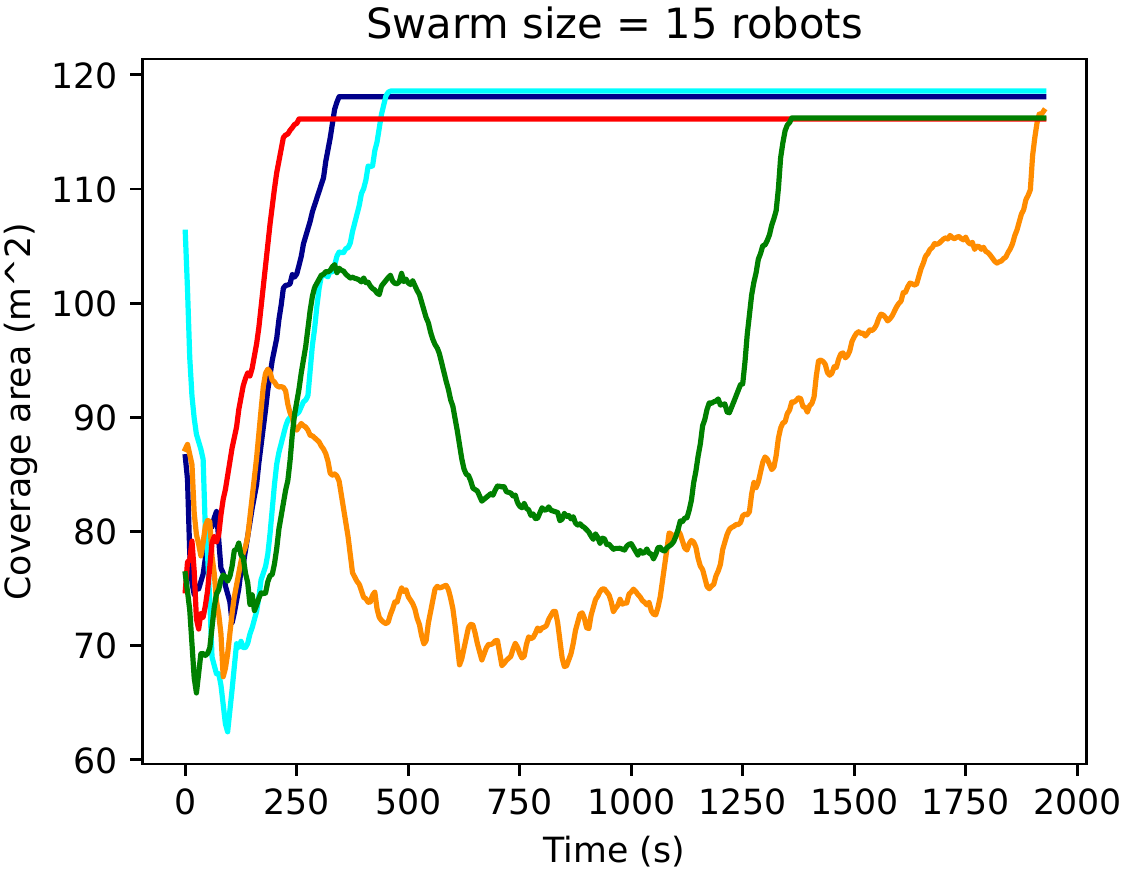}
    \includegraphics[width=0.32\textwidth]{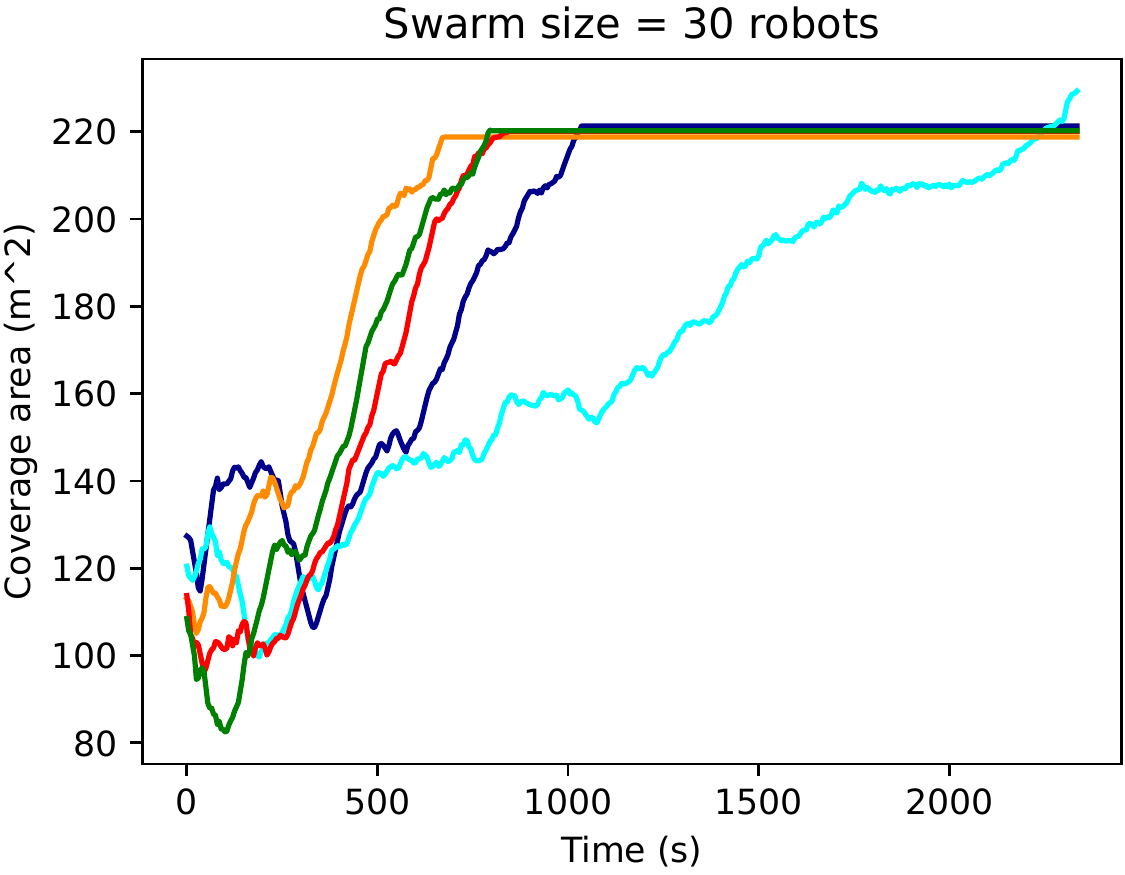}
    \includegraphics[width=0.32\textwidth]{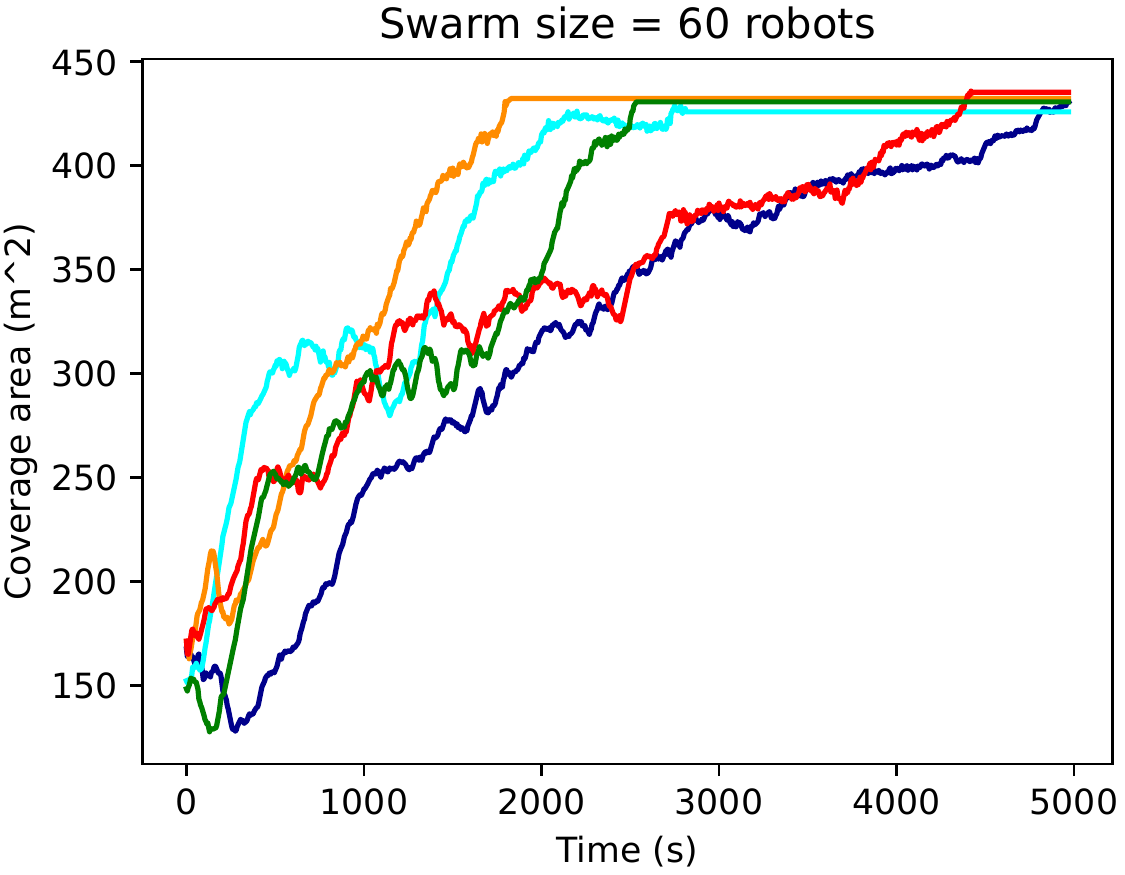}
    \caption{Line formation for 15, 30, and 60 robots column wise. From top to bottom in each column the evolution of $\lambda_2$ of the spanning tree and coverage area over time is plotted.}
\label{fig:line_formation}
\end{figure*}%
\begin{figure*}[h]
    \centering
    \includegraphics[width=0.325\textwidth]{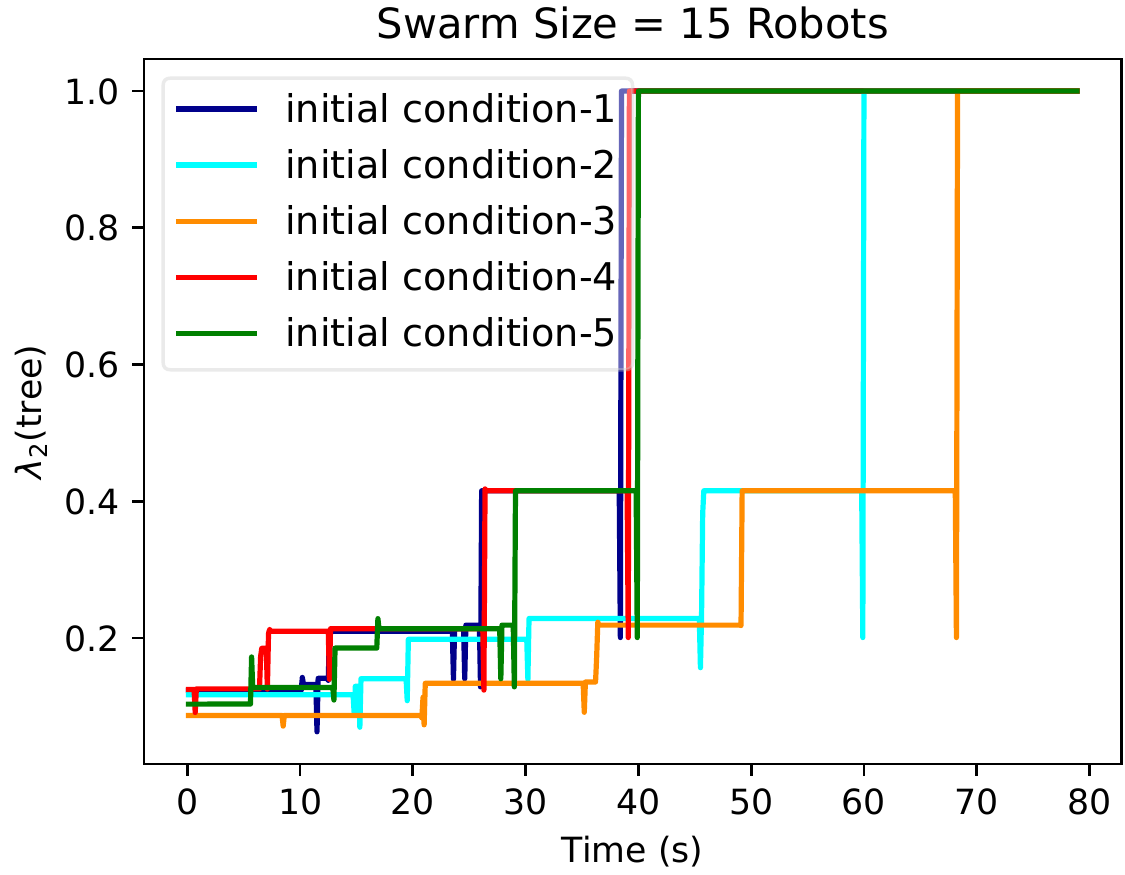}
    \includegraphics[width=0.325\textwidth]{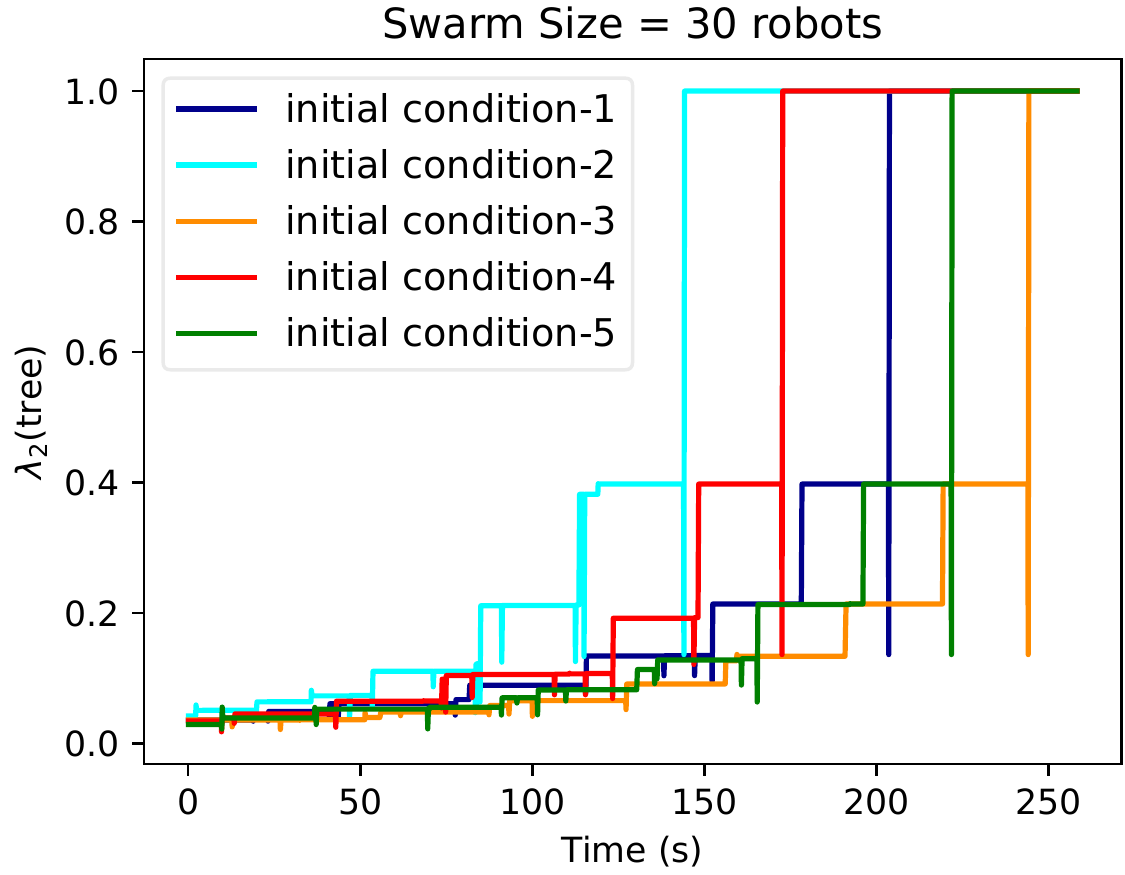}
    \includegraphics[width=0.325\textwidth]{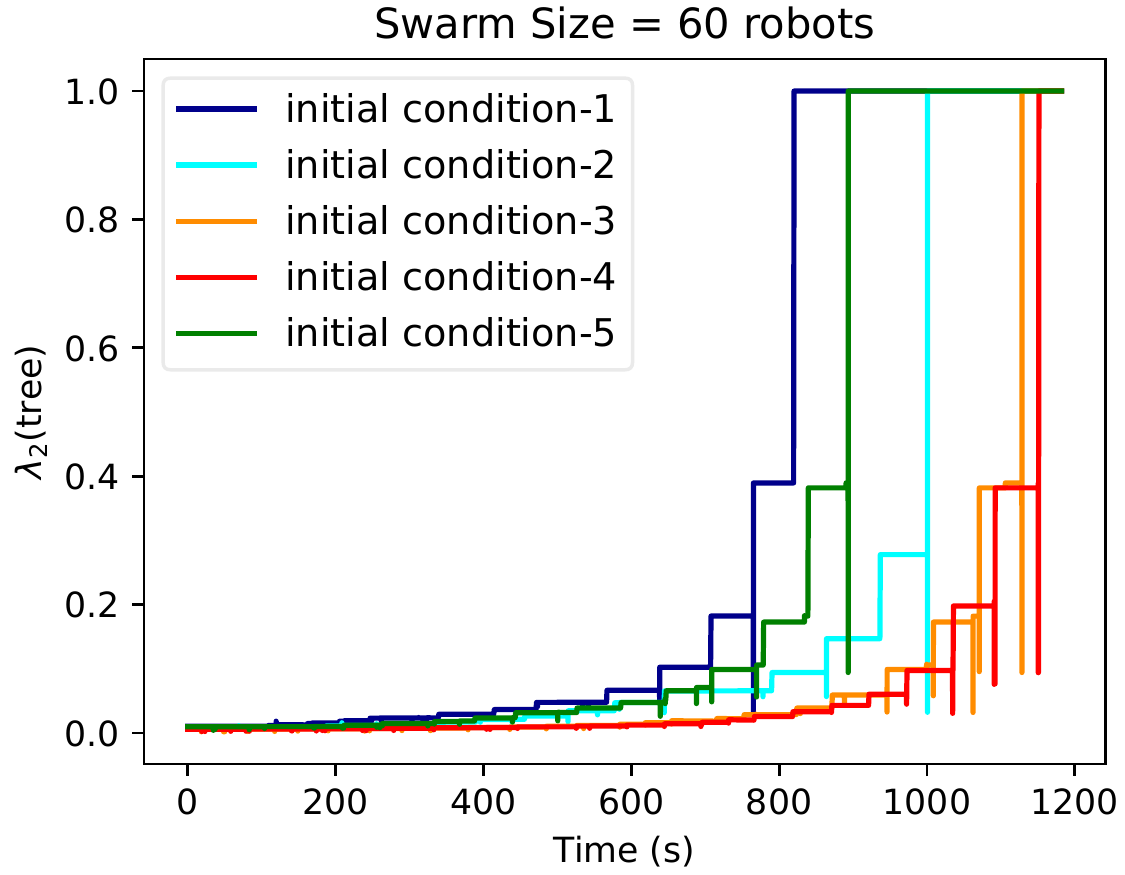}
    \includegraphics[width=0.32\textwidth]{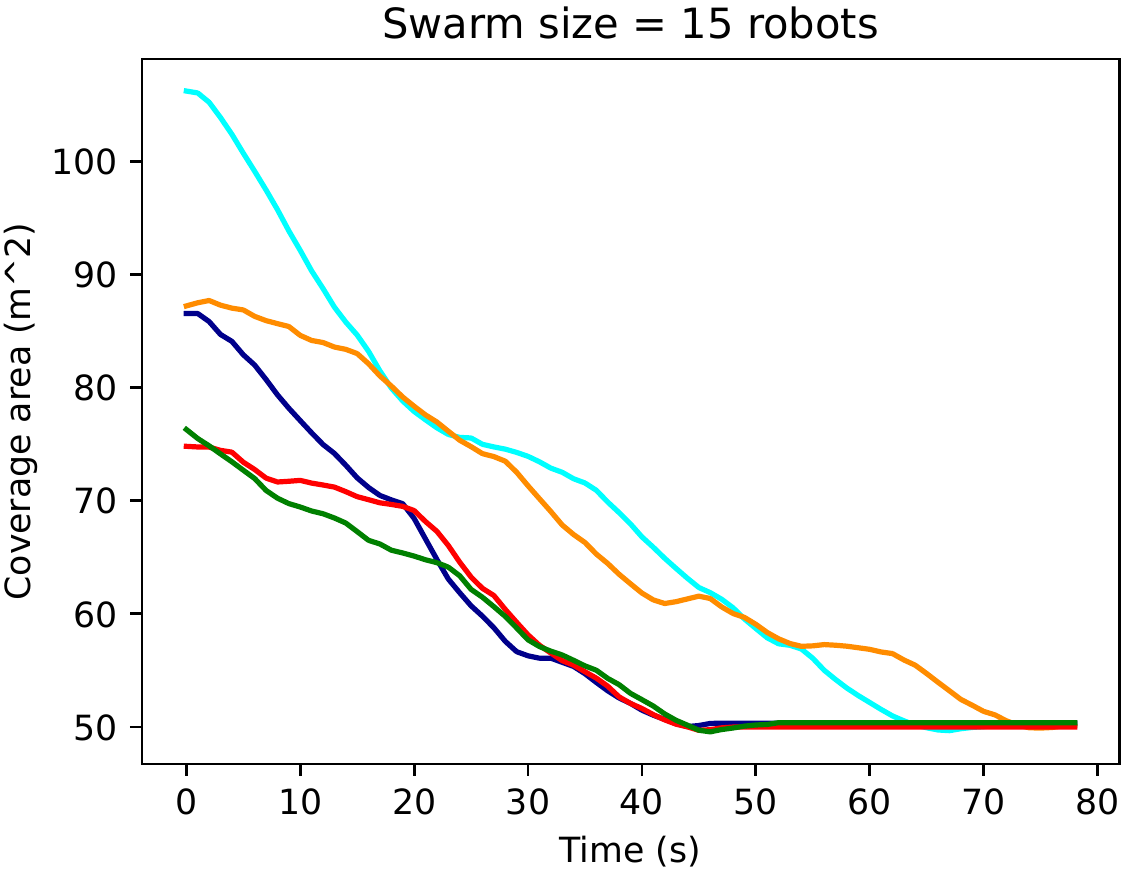}
    \includegraphics[width=0.32\textwidth]{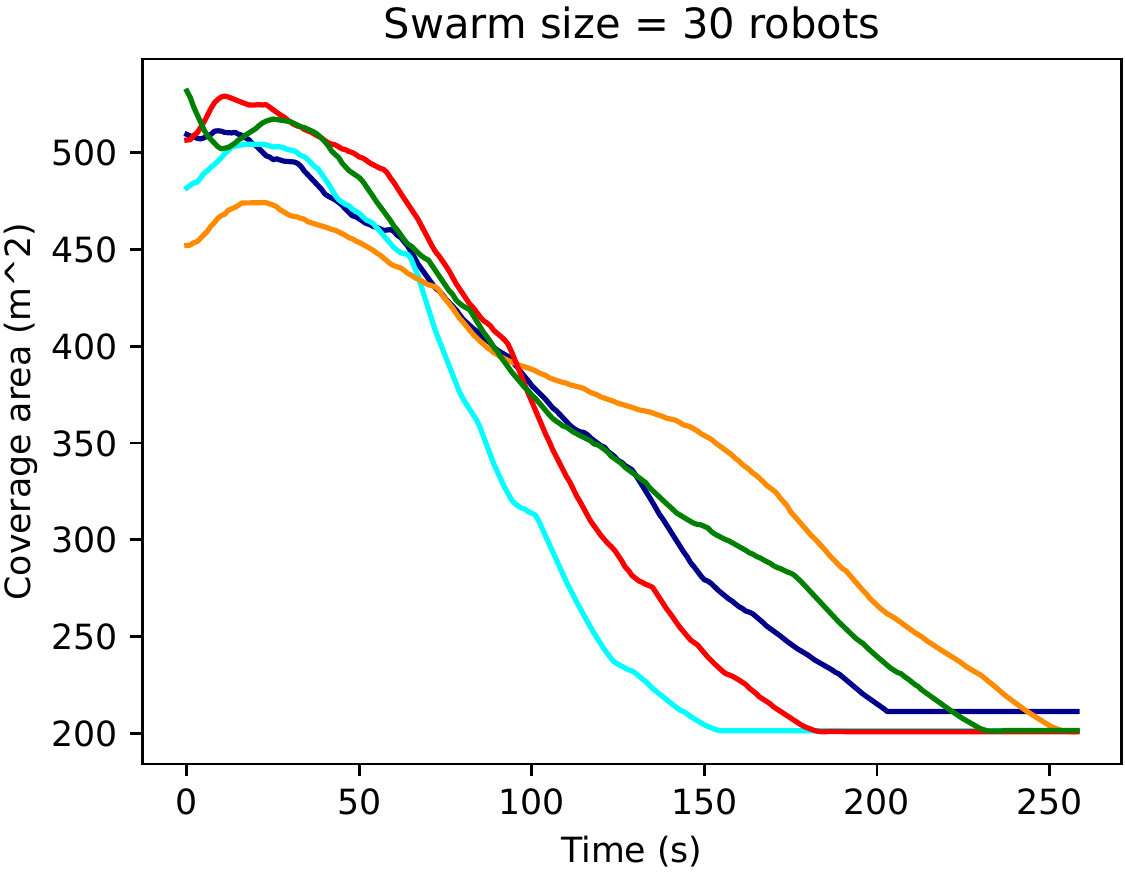}
    \includegraphics[width=0.32\textwidth]{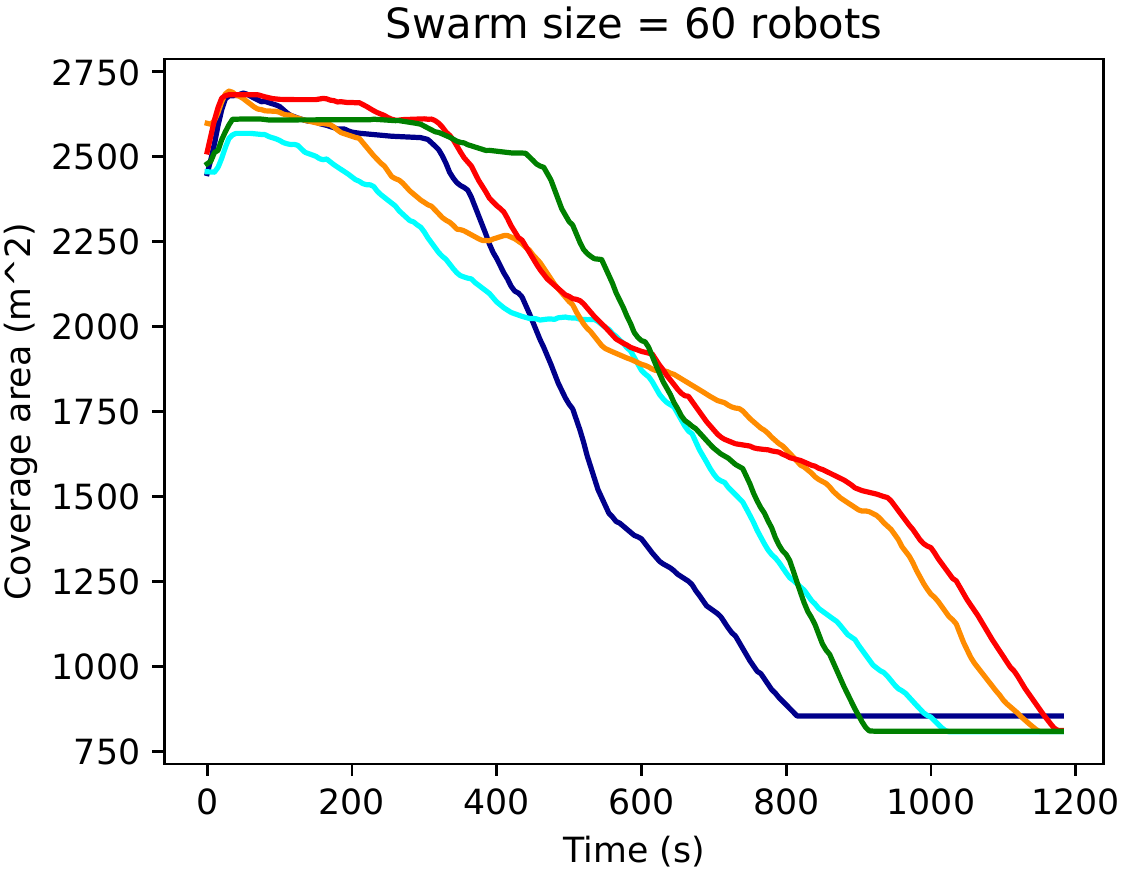}
    \caption{Star formation for 15, 30, and 60 robots column wise. From top to bottom in each column the evolution of $\lambda_2$ of the spanning tree and coverage area over time are plotted.}
\label{fig:star_formation}
\vspace{-2em}
\end{figure*}%

\section{CONCLUSION}
The possibility of introducing a complete set of connectivity-aware local topology manipulator operations \RvTypoIcra{have} been proven to be achievable in this paper. This result has shown that estimating a global index of connectivity is not necessary for topology manipulation in its most flexible sense. This is important because estimating global connectivity indices in a decentralized fashion is sluggish and not scalable. As an application for this local method, transforming a swarm in such a way that achieves extreme contradicting properties (i.e., maximum consensus rate and coverage area) was shown to be possible. Future research in this area can consist of finding another complete set of operations with better performance on certain applications and applying this local method to other swarm missions.

\bibliographystyle{IEEEtran}
\bibliography{ref}

\end{document}